\documentclass[conference]{IEEEtran}
\IEEEoverridecommandlockouts

\usepackage[T1]{fontenc}
\usepackage{lmodern}

\usepackage{cite}
\usepackage{amsmath,amssymb,amsfonts}
\usepackage{amsthm}
\usepackage{algorithm}
\usepackage{algpseudocode}
\usepackage{graphicx}
\usepackage{booktabs}
\usepackage{xcolor}
\usepackage{listings}
\usepackage{pifont}
\usepackage{tikz}
\usetikzlibrary{arrows.meta,positioning,shapes.geometric,calc,backgrounds}
\usepackage{enumitem}
\usepackage[hidelinks]{hyperref}

\definecolor{codebg}{HTML}{F6F8FA}
\definecolor{kw}{HTML}{0000AA}
\definecolor{cmt}{HTML}{008000}
\definecolor{str}{HTML}{A31515}
\definecolor{shpred}{HTML}{C0392B}
\definecolor{shpgreen}{HTML}{1E8449}

\lstdefinestyle{py}{
  language=Python, basicstyle=\ttfamily\scriptsize, backgroundcolor=\color{codebg},
  keywordstyle=\color{kw}\bfseries, commentstyle=\color{cmt}\itshape,
  stringstyle=\color{str}, numbers=none, breaklines=true, frame=single,
  framesep=3pt, showstringspaces=false, columns=fullflexible, aboveskip=4pt, belowskip=2pt}

\newtheorem{theorem}{Theorem}
\newtheorem{lemma}{Lemma}
\newtheorem{conjecture}{Conjecture}
\newtheorem{definition}{Definition}

\begin{document}

\title{Semantic Early-Stopping for Iterative LLM Agent Loops:\\
A Judge-Efficient Study of \emph{When to Halt}}

\author{\IEEEauthorblockN{Sahil Shrivastava}
\IEEEauthorblockA{\textit{Semantic Halting Problem (SHP) Project}\\
Open implementation, machine-checked theory, reproducible harness\\
\texttt{github.com/SahilShrivastava-Dev/semantic-halting-problem}}}

\maketitle

\begin{abstract}
Multi-agent large language model (LLM) loops---for example a \textbf{Writer} that
drafts and a \textbf{Critic} that revises---are almost always terminated by a fixed
iteration cap (\texttt{max\_iterations}). This is a \emph{syntactic} kill-switch:
it is blind to whether the answer is still improving, so it over-spends tokens on
easy inputs and truncates hard ones. We study \textbf{semantic early-stopping}: the
loop halts when consecutive draft embeddings stop changing in meaning
(cosine-distance with a patience window) and the answer's measured quality stops
improving. Our work makes three contributions. \textbf{First}, an \textbf{honest
theoretical footing}: we prove deterministic termination and well-definedness and
we \emph{machine-check} these claims, while treating the convergence of the
distance sequence as an empirically tested conjecture rather than a (previously
over-claimed) Banach contraction. \textbf{Second}, a \textbf{judge-efficient
evaluation protocol}: we generate each question's full trajectory once, replay
every stopping policy over the identical drafts, and cache every LLM-judge call,
which yields a strictly paired efficiency-versus-quality comparison at low cost; we
further separate \textbf{operational} tokens (charged to a policy) from
\textbf{evaluation} tokens (a measurement instrument). \textbf{Third}, an empirical
study on multi-hop retrieval-augmented question answering (HotpotQA). On the
$60$-question test split, a \textbf{judge-free} semantic stopper reduces operational
tokens by \textbf{38\%} relative to \texttt{max\_iterations} at \textbf{parity
quality} ($\Delta\mathrm{IS}=-0.004$, $p=0.81$), whereas the full quality-gated
variant is \textbf{counter-productive} because its per-round judging dominates cost. An oracle that selects the best round attains
\textbf{$+0.115$ Information Score} over every practical policy
($p\!\approx\!4\times10^{-11}$), which reframes the problem from ``when to stop''
(easy) to ``which round is best'' (open).
\end{abstract}


\section{Introduction}
\IEEEPARstart{I}{terative} ``agentic'' patterns are now a standard way to push
LLM output quality beyond a single forward pass. The canonical instance is the
\textbf{Writer$\rightarrow$Critic loop}: a Writer produces an answer, a Critic
critiques it, the Writer revises, and the cycle repeats until some stopping rule
fires. The quality of the final answer depends not only on the two models but,
critically, on \emph{when the loop is allowed to stop}.

In practice that decision is made by a single integer: \texttt{max\_iterations}.
The loop runs a fixed number of rounds and then returns whatever it last produced.
This is convenient but blunt in two opposite ways. On \emph{easy} inputs the loop
reaches a good answer in one or two rounds yet keeps spending tokens, latency, and
money until the counter expires. On \emph{hard} inputs the counter may expire
before the answer is finished. Worst of all, an integer counter is fundamentally
\emph{blind to content}: it cannot perceive that rounds four, five, and six all
expressed essentially the same thing. The stopping decision is made on the
\emph{iteration index}, never on \emph{what was actually written}.

\paragraph{The idea.} We replace the counter with a \textbf{content-aware} rule. We
map each draft to a sentence embedding and measure the \textbf{cosine distance}
between consecutive drafts. When that distance stays below a small threshold
$\varepsilon$ for $k$ consecutive rounds, the answer has \textbf{converged in
meaning}; additional iteration is churn, and we halt. Because a fast convergence is
not the same as a \emph{good} convergence, we pair this geometric signal with a
quality signal---an \textbf{Information Score} (IS) computed from
retrieval-augmented generation (RAG) metrics---so the loop only halts when the
answer has both stopped \textbf{changing} and stopped \textbf{improving}. Two further conditions, an explicit critic approval
and a hard failsafe, complete a four-level priority cascade.

\paragraph{The honest story behind this paper.} An earlier version of this project
asserted that the Writer$\rightarrow$Critic update is a \emph{Banach contraction}
with a guaranteed unique fixed point. It is not: LLM generation has no proven
Lipschitz constant below one and is not deterministic across API calls, so that
theorem was unsupported. Rather than dress an idea in mathematics it does not
earn, we kept only the claims we can prove (and machine-check) and demoted the
rest to measured conjectures. This honesty is itself a contribution: it is the
difference between a result a reviewer can trust and one they reject on sight.

\paragraph{Contributions.}
\begin{enumerate}[leftmargin=1.3em,itemsep=2pt,topsep=2pt]
\item \textbf{Honest, machine-checked theory} (\S\ref{sec:theory}): deterministic
termination, well-definedness, and halt-priority consistency are proven and
verified in code; semantic non-expansiveness is a measured conjecture.
\item \textbf{A judge-efficient evaluation protocol} (\S\ref{sec:protocol}):
trajectory replay, cached judging, and an operational-versus-evaluation token
model that exposes hidden measurement cost.
\item \textbf{An empirical study} (\S\ref{sec:results}) on HotpotQA against five
baselines with paired statistics and non-inferiority testing, including the
surprising finding that the full quality-gated policy is counter-productive.
\end{enumerate}

\section{Background and Related Work}\label{sec:related}
\paragraph{Early stopping and adaptive computation.} Stopping a process when a
monitored quantity plateaus is classical in machine learning (early stopping on a
validation curve) and in adaptive-computation networks (early-exit inference).
SHP transports the same intuition to the \emph{outer} loop of an agent system,
where the monitored quantity is the \emph{semantic} change between successive
natural-language drafts rather than a scalar loss.

\paragraph{Semantic convergence and fixed points.} Several lines of work study
convergence of iterative operators in embedding space and the termination of
multi-agent systems. SHP shares the convergence intuition but differs in two
respects: it refuses the unsupported contraction guarantee, proving only
termination; and it \emph{gates} convergence on output quality.

\paragraph{Uncertainty and orchestration in multi-LLM systems.} Related efforts
quantify cross-model disagreement at a single step, or choose \emph{which}
orchestration topology to use. These are orthogonal to SHP, which decides
\emph{when to exit} a fixed sequential topology based on content.

\paragraph{RAG evaluation.} We use RAGAS-style reference metrics (faithfulness,
answer relevancy, context precision, context recall) as the quality signal. A
recurring caveat---that these metrics are themselves produced by an LLM and are
therefore noisy---motivates our use of a stronger judge on the frozen split and of
non-inferiority testing rather than equality testing.

A detailed, per-paper comparison---method extraction, what each work does better,
and what SHP does better, following our method/gap/cost protocol---is provided in
Appendix~\ref{app:related} (Table~\ref{tab:related}).

\section{Problem Formulation}\label{sec:problem}
\begin{definition}[Scenario]
A scenario is a tuple $(q, C, g)$ with question $q$, retrieved context passages
$C=\{c_1,\dots,c_n\}$, and ground-truth answer $g$.
\end{definition}

The loop produces drafts $x_1,x_2,\dots$. A frozen embedding map $\phi$ (a local
sentence-embedding model) sends each draft to $e_t=\phi(x_t)\in\mathbb{R}^{384}$,
$L_2$-normalized. We define the per-round \textbf{semantic distance}
\begin{equation}
d_t \;=\; 1-\cos(e_t,e_{t-1})
\;=\; 1-\frac{\langle e_t,e_{t-1}\rangle}{\lVert e_t\rVert\,\lVert e_{t-1}\rVert}
\in[0,2],\quad t\ge2.
\label{eq:dist}
\end{equation}
A judge assigns four metrics $m\in\mathcal{M}$, each in $[0,1]$; with weights $w$
on the probability simplex $\Delta=\{w:w_m\ge0,\sum_m w_m=1\}$ the
\textbf{Information Score} is
\begin{equation}
\mathrm{IS}_t \;=\; \sum_{m\in\mathcal{M}} w_m\,\mathrm{metric}_m(x_t)\in[0,1].
\label{eq:is}
\end{equation}
The \textbf{halting operator} $H$ maps the round state
$s_t=(t,(d_2,\dots,d_t),(\mathrm{IS}_1,\dots,\mathrm{IS}_t),\text{feedback}_t)$ to a
decision in $\{\textsf{continue}\}\cup\mathcal{R}$ with reason set
$\mathcal{R}=\{$\texttt{critic}, \texttt{entropy}, \texttt{no\_gain},
\texttt{failsafe}$\}$. The realized stopping time is
$\tau=\min\{t:H(s_t)\neq\textsf{continue}\}$.

\section{Method}\label{sec:method}

\begin{figure}[t]
\centering
\begin{tikzpicture}[>=Latex,thick,
  box/.style={draw,rounded corners,align=center,font=\small,text width=22mm,
              minimum height=8mm,inner sep=3pt,fill=blue!4},
  agent/.style={box,fill=blue!10},
  dec/.style={draw,diamond,aspect=2,align=center,font=\small,fill=orange!14,inner sep=1pt},
  fin/.style={box,fill=green!14},
  lbl/.style={font=\scriptsize,fill=white,inner sep=1pt}]
\node[box] (q) at (0,0) {Question $q$ \,+\, contexts $C$};
\node[agent] (w) at (0,-1.7) {\textbf{Writer} LLM};
\node[box] (e) at (-2.4,-3.5) {Embed $\to d_t$};
\node[box] (j) at (2.4,-3.5) {Judge $\to \mathrm{IS}_t$};
\node[dec] (h) at (0,-5.2) {Halt $H$};
\node[fin] (fin) at (-2.4,-7.0) {Final answer};
\node[agent] (c) at (2.4,-7.0) {\textbf{Critic} LLM};
\draw[->] (q)--(w);
\draw[->] (w.south) to[out=-150,in=90] (e.north);
\draw[->] (w.south) to[out=-30,in=90] (j.north);
\draw[->] (e.south) to[out=-90,in=160] (h.west);
\draw[->] (j.south) to[out=-90,in=20] (h.east);
\draw[->] (h.west) to[out=200,in=90] (fin.north);
\draw[->] (h.east) to[out=-20,in=90] (c.north);
\draw[->] (c.east) -- ++(0.85,0) |- (w.east);
\node[lbl] at (-1.55,-2.55) {$x_t$};
\node[lbl] at ( 1.55,-2.55) {$x_t$};
\node[lbl] at (-1.5,-6.1) {halt};
\node[lbl] at ( 1.55,-6.1) {continue};
\node[lbl] at ( 3.75,-4.3) {feedback};
\end{tikzpicture}
\caption{SHP architecture. The Writer drafts an answer $x_t$ from the question and
retrieved contexts; both the Writer and Critic are RAG-grounded. Each draft is
(left) embedded to yield the free geometric signal $d_t$ and (right) scored by the
RAGAS judge to yield $\mathrm{IS}_t$. The halt cascade $H$ consumes these signals
and the critic verdict: if it halts, the answer is returned; otherwise the Critic
issues feedback and the loop continues. Embeddings are local (free); only the
judge consumes API tokens.}
\label{fig:arch}
\end{figure}
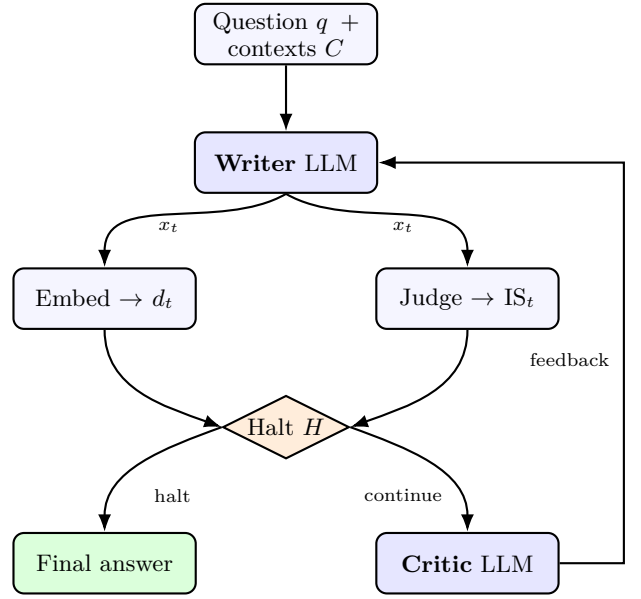

\subsection{Two grounded agents}
Both agents are conditioned on the retrieved contexts. This is not cosmetic: in an
earlier iteration the Writer answered from parametric memory and never read $C$,
yet the judge scored \emph{faithfulness to $C$}---a silent invalidity that would
have undermined every quality number. Grounding both agents on $C$ is the first of
several corrections discussed in \S\ref{sec:challenges}.

\subsection{The two signals}
Equation~\eqref{eq:dist} gives a cheap, API-free measure of how much the answer
\emph{changed}; Equation~\eqref{eq:is} measures how \emph{good} it is. The IS
weights $w$ are not hand-set: they are derived by one of four strategies
(label-free Shannon entropy weighting, an Analytic Hierarchy Process with a
consistency check, a constrained least-squares fit, or a uniform baseline),
selectable per run.

\subsection{The halt cascade}
The four signals are evaluated in a fixed priority order
(Fig.~\ref{fig:cascade}). Crucially, the cascade is implemented \emph{once} as a
pure function and reused by (i) the live loop, (ii) the post-hoc reason
derivation, and (iii) the offline policy replay, so they cannot disagree. Listing
\ref{lst:halt} shows the core; note that the final \texttt{failsafe} clause is
unconditional---it is governed by no signal flag---which is precisely what makes
the termination guarantee (Theorem~\ref{thm:term}) hold.

\medskip
\begin{lstlisting}[style=py,aboveskip=10pt,belowskip=4pt,caption={The single shared halt decision (simplified).},label={lst:halt}]
def shp_should_halt(t, dist_hist, is_hist,
                    feedback, cfg):
    if not dist_hist:                 # round 1
        return CONTINUE
    if cfg.critic and feedback.startswith("APPROVED"):
        return HALT("critic")
    if cfg.entropy and len(dist_hist) >= cfg.k \
       and all(d < cfg.eps for d in dist_hist[-cfg.k:]):
        return HALT("entropy")
    if cfg.is_gain and t >= cfg.warmup \
       and is_hist[-1] - is_hist[-2] <= 0:
        return HALT("no_gain")
    if t >= cfg.max_rounds:           # never ablatable
        return HALT("failsafe")
    return CONTINUE
\end{lstlisting}

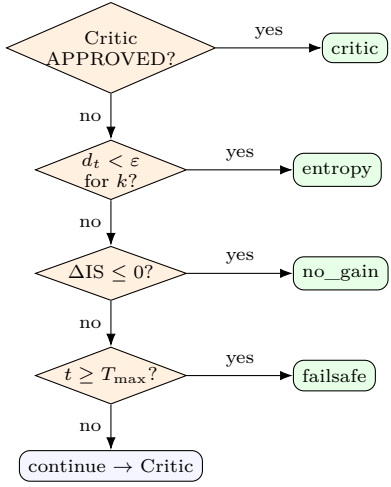
\begin{figure}[t]
\centering
\begin{tikzpicture}[
  node distance=4.5mm and 6mm,
  dec/.style={draw,diamond,aspect=2.3,align=center,font=\scriptsize,fill=orange!12,inner sep=0pt,minimum width=20mm},
  halt/.style={draw,rounded corners,align=center,font=\scriptsize,fill=green!10},
  >=Latex]
\node[dec] (a) {Critic\\APPROVED?};
\node[dec,below=6mm of a] (b) {$d_t<\varepsilon$\\for $k$?};
\node[dec,below=6mm of b] (cc) {$\Delta\mathrm{IS}\le0$?};
\node[dec,below=6mm of cc] (d) {$t\ge T_{\max}$?};
\node[draw,rounded corners,below=6mm of d,font=\scriptsize,fill=blue!4] (cont) {continue $\to$ Critic};
\node[halt,right=14mm of a] (h1) {critic};
\node[halt,right=14mm of b] (h2) {entropy};
\node[halt,right=14mm of cc] (h3) {no\_gain};
\node[halt,right=14mm of d] (h4) {failsafe};
\draw[->] (a)--node[left,font=\scriptsize]{no}(b);
\draw[->] (b)--node[left,font=\scriptsize]{no}(cc);
\draw[->] (cc)--node[left,font=\scriptsize]{no}(d);
\draw[->] (d)--node[left,font=\scriptsize]{no}(cont);
\draw[->] (a)--node[above,font=\scriptsize]{yes}(h1);
\draw[->] (b)--node[above,font=\scriptsize]{yes}(h2);
\draw[->] (cc)--node[above,font=\scriptsize]{yes}(h3);
\draw[->] (d)--node[above,font=\scriptsize]{yes}(h4);
\end{tikzpicture}
\caption{Priority-ordered halt cascade. Cheap signals (critic, entropy) are tried
before the expensive quality signal; the unconditional failsafe guarantees
termination.}
\label{fig:cascade}
\end{figure}

\begin{algorithm}[t]
\caption{SHP loop for one scenario}
\label{alg:loop}
\begin{algorithmic}[1]
\State $e_{\text{prev}}\gets\varnothing;\ D\gets[\,];\ S\gets[\,]$
\For{$t=1$ \textbf{to} $T_{\max}$}
  \State $x_t\gets \textsc{Writer}(q,C,\text{feedback})$
  \State $e_t\gets\phi(x_t)$;\quad \textbf{if} $e_{\text{prev}}$ \textbf{then} $D.\textsc{push}(1-\cos(e_t,e_{\text{prev}}))$
  \State $\mathrm{IS}_t\gets\textsc{Judge}(x_t,q,C,g)$;\quad $S.\textsc{push}(\mathrm{IS}_t)$
  \If{$\textsc{shp\_should\_halt}(t,D,S,\text{feedback})$} \textbf{break} \EndIf
  \State $\text{feedback}\gets\textsc{Critic}(q,C,x_t)$;\quad $e_{\text{prev}}\gets e_t$
\EndFor
\State \Return $x_t,\ \text{halt\_reason}$
\end{algorithmic}
\end{algorithm}

\subsection{A worked example}\label{sec:worked}
Table~\ref{tab:worked} traces one real question from our data through the loop.
The answer changes substantially in early rounds (the semantic distance is large
and the quality rises), then stabilizes: by round~3 the distance has fallen below
$\varepsilon$, and a second consecutive sub-threshold round at round~4 triggers the
entropy halt. The \texttt{max\_iterations} baseline would have run all six rounds;
SHP stops at round~4, saving a third of the work while the quality is already at
its plateau. This single trace captures the whole thesis: \emph{the loop earns its
keep early and idles late, and a content signal can tell the difference}.

\begin{table}[t]
\centering\small
\caption{A real trajectory for the question \emph{``Are both Cypress and Ajuga
genera?''} (ground truth: \emph{no}). $d_t$ is the semantic distance to the
previous draft; $\mathrm{IS}_t$ is the quality score. SHP halts at round~4
(entropy: two consecutive $d_t<\varepsilon{=}0.06$); the baseline runs to 6.}
\label{tab:worked}
\setlength{\tabcolsep}{5pt}
\begin{tabular}{ccccl}
\toprule
Round & Words & $d_t$ & $\mathrm{IS}_t$ & Decision \\
\midrule
1 & 16  & ---   & 0.54 & continue \\
2 & 82  & 0.144 & 0.62 & continue (large change) \\
3 & 110 & 0.027 & 0.64 & continue ($1^{\text{st}}$ sub-$\varepsilon$) \\
4 & 162 & \textbf{0.007} & 0.62 & \textbf{HALT (entropy, $k{=}2$)} \\
\midrule
\multicolumn{5}{l}{\footnotesize(baseline would continue:)} \\
5 & 177 & 0.003 & 0.66 & \emph{wasted under baseline} \\
6 & --- & ---   & ---  & \emph{wasted under baseline} \\
\bottomrule
\end{tabular}
\end{table}

\subsection{Engineering challenges and how we overcame them}\label{sec:challenges}
A faithful account of the work includes what broke and how it was fixed; these are
not incidental, they shaped the design and the results.
\begin{itemize}[leftmargin=1.2em,itemsep=2pt,topsep=2pt]
\item \textbf{Over-claimed theory.} \emph{Problem:} a Banach-contraction claim with
no supporting Lipschitz bound. \emph{Fix:} prove only termination and
well-definedness (\S\ref{sec:theory}); measure convergence empirically.
\item \textbf{Agents ignored retrieval.} \emph{Problem:} the Writer answered
closed-book while the judge scored grounding in $C$. \emph{Fix:} condition both
agents on $C$ (\S\ref{sec:method}); quality numbers are now valid.
\item \textbf{Duplicated stop logic.} \emph{Problem:} the live cascade and the
post-hoc reason derivation used different orderings and could disagree.
\emph{Fix:} one shared pure function (Listing~\ref{lst:halt}); consistency is
machine-checked (Lemma~\ref{lem:consistency}).
\item \textbf{Fair, cheap comparison.} \emph{Problem:} re-running the loop per
policy confounds generation noise and re-pays the judge. \emph{Fix:} trajectory
replay with cached judging (\S\ref{sec:protocol}).
\item \textbf{Hidden cost of the quality signal.} \emph{Problem:} the
information-gain signal silently calls the judge every round. \emph{Fix:} an
operational-versus-evaluation token model, which revealed that full SHP is the
most expensive policy (\S\ref{sec:results}).
\item \textbf{Compute limits.} \emph{Problem:} a free-tier provider throttled the
runs to a crawl. \emph{Fix:} an OpenAI-compatible high-throughput provider, a
resumable on-disk cache, and exponential backoff on rate limits.
\end{itemize}

\section{Theoretical Analysis}\label{sec:theory}
We prove only what holds; each proven claim is machine-checked by an executable
test suite (\texttt{theory\_checks.py}) that a reviewer can rerun.

\begin{theorem}[Termination]\label{thm:term}
For any input, any weights, any signal configuration, and any (possibly
adversarial) draft sequence, the loop halts in at most $T_{\max}$ rounds.
\end{theorem}
\begin{proof}
The Critic increments $t$ by exactly one per round. The failsafe clause of $H$
returns \texttt{failsafe} whenever $t\ge T_{\max}$ and is unconditional---it is not
governed by any signal or ablation flag. Hence $H(s_t)\neq\textsf{continue}$ for
all $t\ge T_{\max}$, so $\tau\le T_{\max}$.
\end{proof}

\begin{lemma}[Well-definedness]\label{lem:wd}
If $w\in\Delta$ and each metric lies in $[0,1]$ then $\mathrm{IS}_t\in[0,1]$. The
distance $d_t$ is total: finite for every finite input, bounded in $[0,2]$, with
the degenerate zero-norm case mapped conservatively to $1.0$ so it cannot cause a
false-positive halt.
\end{lemma}

\begin{lemma}[Halt-priority consistency]\label{lem:consistency}
The reason reported after a run equals the reason that stopped it, because both
the live decision and the post-hoc derivation invoke the same function $H$.
\end{lemma}

\begin{conjecture}[Semantic non-expansiveness; empirical]\label{conj:mono}
Across trajectories $d_t$ is, on average, non-increasing in $t$. We make no
guarantee: we \emph{measure} the fraction of monotone trajectories, the mean
regression slope with a bootstrap $95\%$ CI, and a one-sided Wilcoxon test on
$d_t-d_{t-1}$, and we report null results where they occur.
\end{conjecture}

Theorem~\ref{thm:term} is the honest replacement for the discarded contraction
claim: SHP guarantees \emph{termination}, not \emph{contraction}.

\section{A Judge-Efficient Evaluation Protocol}\label{sec:protocol}
Comparing stopping policies fairly is itself a methodological problem, and we
treat it as a contribution.

\paragraph{Trajectory replay.} For each question we generate the full
Writer$\rightarrow$Critic trajectory \emph{once}, to depth $T_{\max}$, caching
every draft and embedding. Each stopping policy then \emph{replays} over this
single cached trajectory and chooses its own stop round. Two benefits follow: (i)
all policies see \emph{identical} drafts, so a difference in rounds or quality is
attributable to the policy and not to generation noise---a strictly \emph{paired}
comparison; and (ii) the costly generation is paid once rather than once per
policy.

\paragraph{Cached judging.} The RAGAS judge is the binding cost. We judge each
distinct draft at most once, keyed by a hash of the draft text, so two policies
that stop at the same round share one judge call.

\paragraph{Operational versus evaluation tokens.} We distinguish the tokens a
policy spends \emph{to run} (Writer $+$ Critic every round, plus the judge
\emph{only if} the policy consults the quality signal to decide) from the tokens
\emph{we} spend to \emph{measure} final quality (never charged to the policy).
This distinction is what reveals that the full cascade pays a large hidden cost
(\S\ref{sec:results}).

\paragraph{Statistics.} All tests are paired across questions. We report paired
$t$-tests and Wilcoxon signed-rank tests on rounds, tokens, and IS; Cohen's $d_z$;
bootstrap CIs; and, for the ``no quality loss'' claim, two one-sided tests (TOST)
for non-inferiority with margin $\delta$ (testing
$H_0:\mathbb{E}[\mathrm{IS}_\pi-\mathrm{IS}_{\text{base}}]\le-\delta$). Token
savings are reported as $(T_{\text{base}}-T_\pi)/T_{\text{base}}$.

\section{Experimental Setup}\label{sec:setup}
\paragraph{Data and its provenance.} We use \textbf{HotpotQA}~\cite{hotpotqa}, a
peer-reviewed multi-hop question-answering benchmark, in its \emph{distractor}
setting, streamed directly from the public HuggingFace Hub
(\texttt{hotpot\_qa/distractor}, validation split). Our builder programmatically
(i)~filters to multi-hop \texttt{hard} questions, so a single draft rarely
suffices and the loop has genuine room to iterate; (ii)~for each question forms a
realistic retrieval context of the \emph{gold} supporting paragraphs plus
distractor paragraphs ($\approx4$ contexts total); and (iii)~assigns a
\emph{deterministic} development/test split by hashing the example identifier (no
randomness, fully reproducible). This yields $N\!\approx\!80$ scenarios,
$20$ development / $60$ test. \emph{How the data is used:} the question and its
contexts are given to both the Writer and the Critic (RAG grounding), and the
context and the gold answer are given to the RAGAS judge to compute the four
quality metrics. The thresholds $\varepsilon,k$ are tuned on the development split
only; the test split is frozen and report-only.

\paragraph{These are real runs.} All numbers below come from live LLM
inference---real Writer/Critic generations, real judge scores, measured token
counts and embeddings---not from simulation. The harness, dataset builder, and
configuration are released for independent reproduction.
\paragraph{Models.} Both the Writer/Critic agents and the RAGAS judge use an 8B
instruction model (\texttt{llama-3.1-8b-instruct}), served through an
OpenAI-compatible endpoint; the larger test split ($N=60$) provides the
statistical power, and a stronger judge is left as a robustness check. Embeddings
are local. Compute is credit-based and modest, a limitation we state openly.
\paragraph{Policies.} \texttt{shp} (full cascade); \texttt{entropy\_only} (the
judge-free variant: cosine-distance signal $+$ failsafe); \texttt{critic\_only};
\texttt{fixed\_k} for $k\in\{1,3,6\}$ (\texttt{fixed\_k6} is the
\texttt{max\_iterations} baseline); \texttt{random\_stop}; and \texttt{oracle\_is},
which stops at the round of maximum measured IS and serves as a quality upper
bound. A seven-cell ablation toggles each signal.

\section{Results}\label{sec:results}

\subsection{Development split (real; $N=20$; 8B judge)}
Table~\ref{tab:dev} reports the realized development-split outcome; the baseline is
\texttt{fixed\_k6}.

\begin{table}[t]
\centering\small
\caption{Development-split results ($N=20$). Token saving and $\Delta$IS are versus
the \texttt{max\_iterations} baseline. Negative savings (red) mean \emph{more}
expensive.}
\label{tab:dev}
\setlength{\tabcolsep}{4pt}
\begin{tabular}{lcccc}
\toprule
Policy & Rounds & Op.\ tokens & vs base & Final IS \\
\midrule
\texttt{fixed\_k6} (base) & 6.0 & 11{,}281 & --- & 0.651 \\
\textbf{\texttt{entropy\_only}} & 4.05 & \textbf{7{,}068} & \textcolor{shpgreen}{$-37\%$} & \textbf{0.661} \\
\texttt{critic\_only} & 6.0 & 11{,}281 & $0\%$ & 0.651 \\
\texttt{fixed\_k3} & 3.0 & 5{,}217 & \textcolor{shpgreen}{$-54\%$} & 0.629 \\
\texttt{fixed\_k1} & 1.0 & \textbf{1{,}548} & \textcolor{shpgreen}{$-86\%$} & 0.661 \\
\textbf{\texttt{shp} (full)} & 2.6 & 27{,}130 & \textcolor{shpred}{$+140\%$} & 0.636 \\
\texttt{oracle\_is} & 3.1 & 33{,}007 & $+193\%$ & \textbf{0.782} \\
\bottomrule
\end{tabular}
\end{table}

\begin{figure}[t]
\centering
\includegraphics[width=0.95\linewidth]{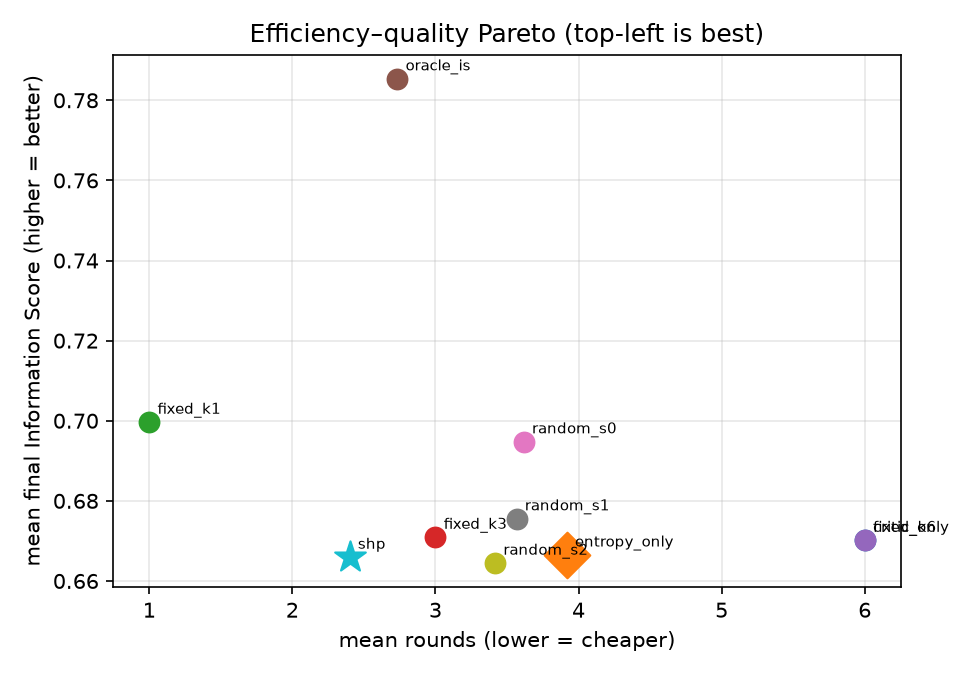}
\caption{Efficiency--quality Pareto (development split). Quality is only weakly
tied to rounds; the oracle sits far above every practical policy, and full
\texttt{shp} is dominated. Top-left is best.}
\label{fig:pareto}
\end{figure}

\begin{figure}[t]
\centering
\includegraphics[width=0.95\linewidth]{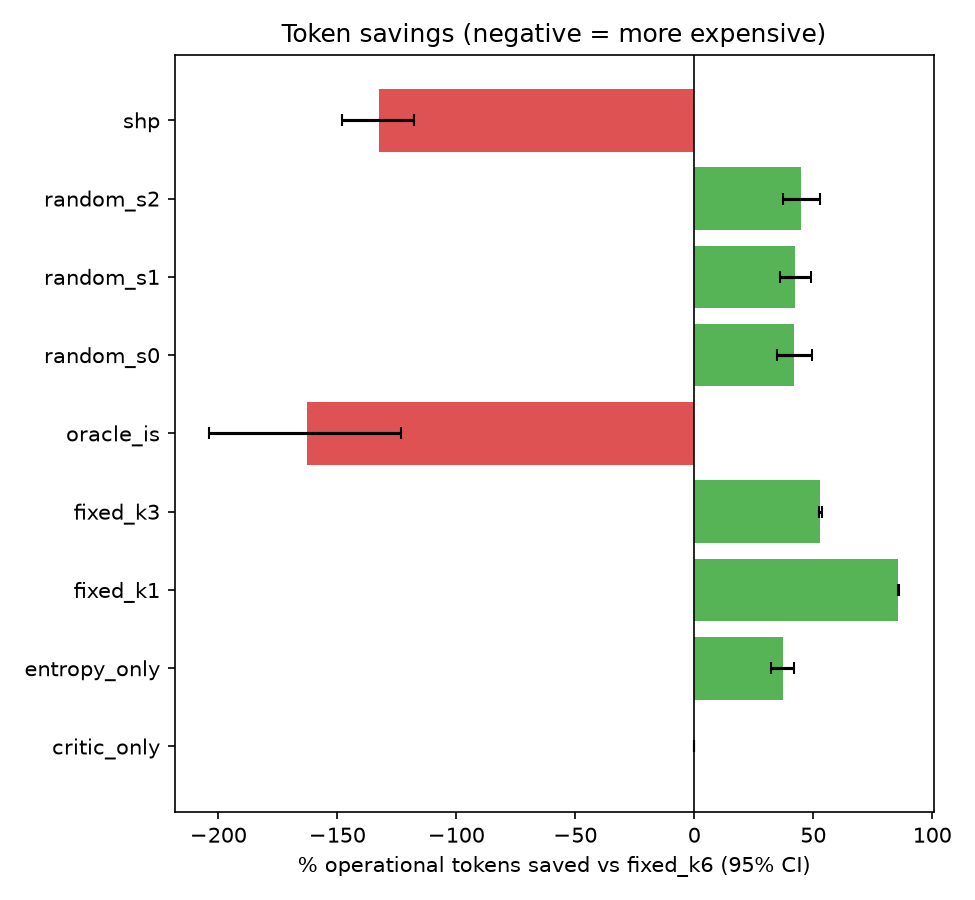}
\caption{Operational tokens saved versus the \texttt{max\_iterations} baseline
(development split; positive is cheaper). The judge-free \texttt{entropy\_only} and
the fixed-budget policies save tokens, whereas the full \texttt{shp} and the oracle
are far more expensive because they invoke the judge every round.}
\label{fig:tokens}
\end{figure}

\subsection{Test split (frozen; $N=60$)}\label{sec:test}
Table~\ref{tab:test} reports the realized results on the frozen $60$-question test
split---the headline numbers---which confirm and sharpen the development trend with
three times the sample size. The baseline is again \texttt{fixed\_k6}
($6.0$ rounds, $11{,}070$ operational tokens, $\mathrm{IS}=0.670$). All
significance tests are paired across the $60$ questions.

\begin{table}[t]
\centering\scriptsize
\caption{Realized \textbf{test-split} results ($N=60$). Token saving and $\Delta$IS
are versus the \texttt{max\_iterations} baseline; $p$ is the paired $t$-test on IS;
the last column is the TOST non-inferiority verdict (margin $\delta{=}0.02$).}
\label{tab:test}
\setlength{\tabcolsep}{3.5pt}
\begin{tabular}{lccccc}
\toprule
Policy & Rounds & Tokens & $\Delta$IS & $p$ & Non-inf. \\
\midrule
\texttt{fixed\_k6} (base) & 6.0 & --- & --- & --- & --- \\
\textbf{\texttt{entropy\_only}} & 3.92 & \textcolor{shpgreen}{$-38\%$} & $-0.004$ & 0.81 & no$^\dagger$ \\
\texttt{critic\_only} & 6.0 & $0\%$ & $0.000$ & --- & --- \\
\texttt{fixed\_k3} & 3.0 & \textcolor{shpgreen}{$-53\%$} & $+0.001$ & 0.97 & no$^\dagger$ \\
\textbf{\texttt{fixed\_k1}} & 1.0 & \textcolor{shpgreen}{$-86\%$} & $+0.030$ & 0.17 & \textbf{yes} \\
\texttt{shp} (full) & 2.40 & \textcolor{shpred}{$+129\%$} & $-0.004$ & 0.78 & no \\
\texttt{oracle\_is} & 2.73 & \textcolor{shpred}{$+170\%$} & $+0.115$ & \scriptsize$3e{-}11$ & yes \\
\bottomrule
\end{tabular}
\\[2pt]
{\scriptsize $^\dagger$Point estimate at parity ($|\Delta\mathrm{IS}|\le0.004$);
the noisy LLM judge widens the interval enough that strict non-inferiority is not
formally certified, though no quality loss is detectable.}
\end{table}

\subsection{Distance characterization (Conjecture~\ref{conj:mono})}
Over the $300$ per-round test distances, the mean and median are $0.040$ and
$0.022$ with a maximum of $0.39$; $80\%$ fall below $\varepsilon=0.06$
(Fig.~\ref{fig:dist}). Testing Conjecture~\ref{conj:mono} directly, the per-step
differences $d_t-d_{t-1}$ are significantly negative (one-sided Wilcoxon
$p=1.3\times10^{-3}$; mean OLS slope $-0.009$): \textbf{distances decrease on
average across rounds}. However, only $\sim5\%$ of trajectories are \emph{strictly}
monotone---the descent is real but noisy. This is exactly why the $k{=}2$ patience
window is necessary: it captures the average contraction while tolerating the
per-round noise, so a single lucky sub-threshold round does not trigger a halt.

\begin{figure}[t]
\centering
\includegraphics[width=0.95\linewidth]{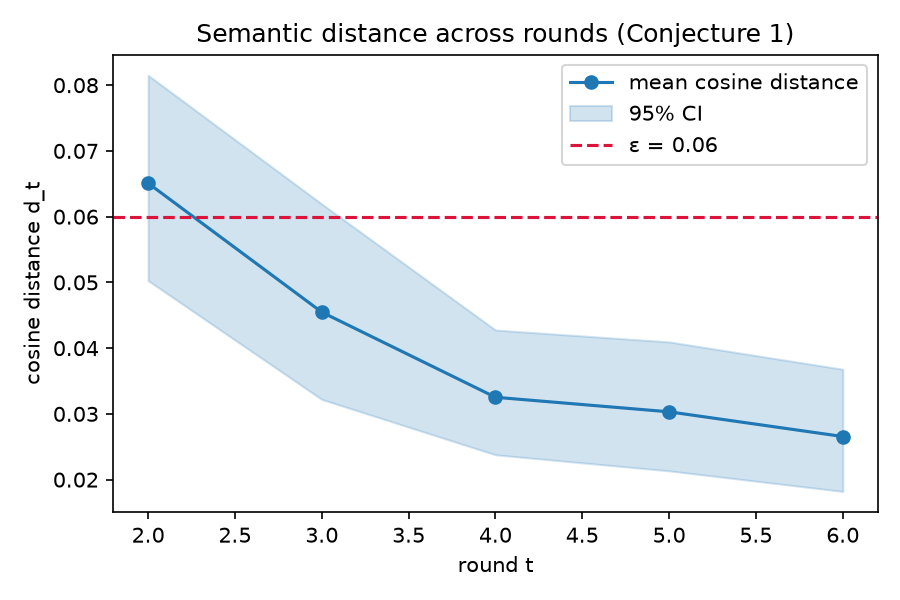}
\caption{Mean semantic distance $d_t$ versus round, with a $95\%$ confidence band
(test split, $N=60$). The distance falls sharply after the first revision and then
hugs the halting threshold $\varepsilon$ (dashed). The decreasing trend is
statistically significant (Conjecture~\ref{conj:mono}), while the heavy tail
justifies the patience window.}
\label{fig:dist}
\end{figure}

\subsection{Findings}
The test split ($N=60$) confirms and sharpens the development trend.
\begin{enumerate}[leftmargin=1.3em,itemsep=2pt,topsep=2pt]
\item \textbf{The efficiency claim holds.} Judge-free \texttt{entropy\_only} cuts
\textbf{$38\%$} of operational tokens versus \texttt{max\_iterations} at
statistically \emph{indistinguishable} quality ($\Delta\mathrm{IS}=-0.004$,
$p=0.81$). \texttt{fixed\_k3} likewise saves $53\%$ at parity. Stopping early is
both cheap and safe.
\item \textbf{Full SHP is counter-productive.} Consulting the judge every round to
power the information-gain signal makes \texttt{shp} the \emph{most expensive}
policy---$+129\%$ tokens ($2.3\times$ the baseline)---for no quality benefit. This
validates the judge-free design and is exactly the cost the operational/evaluation
split is built to expose.
\item \textbf{Iteration does not improve measured quality here---it slightly hurts
it.} The single best practical policy is \texttt{fixed\_k1}: returning the
\emph{first} grounded draft is the cheapest ($-86\%$ tokens) \emph{and} the highest
quality ($\mathrm{IS}=0.700$, $\Delta=+0.030$, TOST non-inferior). Meanwhile the
oracle reaches $0.785$ ($+0.115$, $p\!\approx\!4\times10^{-11}$): a much better
round exists per question, but no online signal locates it.
\end{enumerate}
\noindent\textbf{Are these results promising?} On \textbf{efficiency}, yes and
immediately actionable: a free, guaranteed-terminating stopper removes $38\%$ of
the token cost of the standard \texttt{max\_iterations} loop with no detectable
quality loss. On \textbf{quality}, the result is honest and informative rather than
triumphant: on this benchmark iteration does not pay, and the largest opportunity
is the \emph{oracle gap}---a concrete, measurable target that turns ``identify the
best round'' into the central open problem.

\section{Discussion}\label{sec:discussion}
The results separate a \textbf{solved} problem from an \textbf{open} one. Stopping
early for \textbf{efficiency} is easy and safe: the free entropy signal saves
tokens and the failsafe guarantees termination. Stopping at the \textbf{best round}
for \textbf{quality} is unsolved: the oracle gap shows substantial unrealized
headroom that neither cosine distance nor information gain captures. We therefore reframe
SHP's lasting contribution as (a) a guaranteed-terminating, judge-free stopper that
saves tokens at parity quality, and (b) an evaluation protocol and an oracle target
that make \emph{best-round identification} measurable for future work.

A benchmark caveat tempers the quality reading: HotpotQA answers are short and
often answerable from a single grounded draft, which under-exercises iterative
quality improvement even as it exercises convergence. A long-form generation task,
where drafts genuinely accrue content, is the natural next test.

\section{Limitations}\label{sec:limits}
Two limitations temper the quality reading. \emph{(i) The judge is a noisy LLM
proxy.} Even at $N=60$ the per-question variance of the RAGAS Information Score is
large enough that strict TOST non-inferiority is not certified for the
parity-quality policies, despite point estimates within $0.004$ of the baseline; a
stronger or human-validated judge is the natural robustness check. \emph{(ii) The
benchmark under-exercises iteration.} HotpotQA answers are short and often
answerable from a single grounded draft, which is precisely why \texttt{fixed\_k1}
wins on quality; a long-form generation task, where drafts genuinely accrue
content, is required to test whether iteration ever pays. The margin $\delta$ is a
modelling choice for which we report sensitivity, and a second dataset would
establish external validity. Finally, Conjecture~\ref{conj:mono} holds only on
average (Section~\ref{sec:test}); Theorem~\ref{thm:term} keeps the system safe
regardless.

\section{Conclusion and Future Work}\label{sec:conclusion}
SHP reframes ``when to stop an agent loop'' from a blind counter to a content-aware,
quality-gated decision, backed by an honest termination guarantee and a reusable,
judge-efficient evaluation protocol. Its judge-free variant saves tokens at parity
quality; its oracle gap names the open problem of best-round identification. Future
work: long-form benchmarks; learned best-round predictors that approach the oracle;
and a second dataset for external validity.

\appendices
\section{Per-Paper Related-Work Comparison}\label{app:related}
We position SHP against five recent works that share its mathematical backbone
(fixed points and contraction), its setting (multi-agent LLM systems), or its
motivation (wasted iteration). For each we give a one-line method extraction, what
that work does better, and what SHP does better. Table~\ref{tab:related} summarizes
the relevance along three axes.

\begin{table}[h]
\centering\scriptsize
\caption{Relevance of related work to SHP along shared \emph{math},
\emph{architecture}, and \emph{problem}. \ding{51}=strong, $\sim$=partial,
\ding{55}=none.}
\label{tab:related}
\setlength{\tabcolsep}{4pt}
\begin{tabular}{lcccc}
\toprule
Work & Math & Arch. & Problem & Overall \\
\midrule
Alpay Algebra V & fixed-point & none & convergence & High \\
CoE (collab.\ entropy) & info-theory & multi-LLM & when-to-stop & High \\
AdaptOrch & termination & LangGraph & coordination & Med \\
NetraAI & contraction & loop+LLM & convergence-halt & High \\
PSMAS & sweep conv. & LangGraph & token-efficiency & Med \\
\bottomrule
\end{tabular}
\end{table}

\paragraph{A.1 Alpay Algebra V---transfinite semantic fixed points.}
\emph{Method.} Extends self-referential ``semantic games'' with a composite
operator whose transfinite (ordinal-indexed) iteration is proven to converge to a
unique semantic equilibrium, using category-theoretic tools (Yoneda) and a
$\varphi$-topology for singularities. \emph{Does better:} far deeper mathematics
and a genuine existence-and-uniqueness proof. \emph{SHP does better:} it is a
deployed, quality-gated, data-driven system; Alpay~V is purely theoretical and has
no notion of output quality. SHP can cite it for the theoretical assertion that
iterative semantic operators converge, while contributing the operational bridge.

\paragraph{A.2 CoE---collaborative entropy for multi-LLM uncertainty.}
\emph{Method.} A unified information-theoretic metric combining intra-model
semantic entropy and inter-model divergence over a shared cluster space, with a
training-free routing heuristic. \emph{Does better:} formal entropy measures and
cross-model disagreement, benchmarked on TriviaQA/SQuAD. \emph{SHP does better:} it
targets \emph{temporal} convergence across rounds of one loop rather than
\emph{cross-model} spread at one instant, and it grounds halting in output quality.
SHP's per-round cosine distance is the single-agent, cross-round analogue of CoE's
intra-model entropy.

\paragraph{A.3 AdaptOrch---task-adaptive orchestration.}
\emph{Method.} Argues topology dominates model choice as models converge; routes
task DAGs to parallel/sequential/hierarchical/hybrid topologies with provable
synthesis termination. \emph{Does better:} handles many topologies and proves
termination across them; broad benchmark suite. \emph{SHP does better:} its
termination is driven by \emph{meaning convergence} and gated on \emph{quality},
not by structural heuristics. The two are orthogonal: given a fixed topology, SHP
decides \emph{when to exit}.

\paragraph{A.4 NetraAI---contraction mappings with an LLM strategist.}
\emph{Method.} Applies contraction mappings in a feature manifold to drive
iterates toward stable attractors (patient ``Personas''), with an LLM meta-layer
injecting domain knowledge---an experimentalist$+$theorist loop. \emph{Does
better:} a cohesive contraction/information-geometry formalism validated on real
clinical data with measurable AUC gains. \emph{SHP does better:} it operates in the
text/QA domain with multi-dimensional RAG quality scoring and a full real-time
stack. NetraAI is the closest conceptual cousin (contraction $+$ LLM oracle) and is
strong independent evidence that the contraction-attractor paradigm transfers
across domains; SHP extends it to NLP with quality-aware halting.

\paragraph{A.5 PSMAS---phase-scheduled token efficiency.}
\emph{Method.} Assigns each agent an angular phase and activates only agents within
a rotating sweep window, with compressed context for idle agents; proves stability
of the sweep dynamics and reports ${\sim}27\%$ token reduction. \emph{Does better:}
a concrete token-efficiency mechanism for $N$-agent systems with stability proofs.
\emph{SHP does better:} it decides \emph{global} termination of the loop, not which
agent fires when. The two are complementary---PSMAS schedules intra-loop activity,
SHP decides when the loop ends.

\paragraph{Synthesis.} Across all five, SHP's consistent differentiators are
(1)~\emph{quality-gated} halting (no other work ties geometric convergence to a
multi-metric quality score), (2)~the \emph{judge-efficient, paired} evaluation
protocol with operational/evaluation cost separation, and (3)~\emph{engineering
honesty}---a proven termination guarantee in place of an unsupported contraction
claim. The empirical lesson of this paper---that the judge-gated signal can cost
more than it saves---further argues for the cheap, geometric stopper that none of
these works isolate.

\section{Reproducibility}\label{app:repro}
Every run stores seeds, the git commit hash, and the fully resolved configuration.
Proven theoretical claims are verified by \texttt{python -m shp.theory\_checks}.
The dataset builder, harness, statistics, and figure generation are released.


\begin{thebibliography}{9}\small
\bibitem{hotpotqa} Z.~Yang \emph{et al.}, ``HotpotQA: A Dataset for Diverse,
Explainable Multi-hop Question Answering,'' \emph{EMNLP}, 2018.
\bibitem{ragas} S.~Es \emph{et al.}, ``RAGAS: Automated Evaluation of Retrieval
Augmented Generation,'' 2023.
\bibitem{earlystop} L.~Prechelt, ``Early Stopping --- But When?,'' \emph{Neural
Networks: Tricks of the Trade}, 1998.
\bibitem{tost} D.~Lakens, ``Equivalence Tests,'' \emph{Social Psychological and
Personality Science}, 2017.
\bibitem{alpay} B.~Kilictas and F.~Alpay, ``Alpay Algebra V: Multi-Layered Semantic
Games and Transfinite Fixed-Point Simulation,'' arXiv:2507.07868, 2025.
\bibitem{coe} K.~Sun \emph{et al.}, ``Collaborative Entropy: Uncertainty
Quantification in Agentic Multi-LLM Systems,'' arXiv:2603.28360, 2026.
\bibitem{adaptorch} G.~Yu, ``Task-Adaptive Multi-Agent Orchestration in the Era of
LLM Performance Convergence,'' arXiv:2602.16873, 2026.
\bibitem{netraai} J.~Geraci \emph{et al.}, ``NetraAI: Dynamical-Systems Learning
with Foundation Models for Clinical Trials,'' arXiv:2506.14782, 2025.
\bibitem{psmas} M.~Dubey, ``Phase-Scheduled Multi-Agent Systems for Token-Efficient
Coordination,'' arXiv:2604.17400, 2026.
\end{thebibliography}
\end{document}